\documentclass{article} 
\usepackage{iclr2021_conference,times}

\usepackage{microtype}
\usepackage{graphicx}
\usepackage{subcaption}
\usepackage{booktabs} 

\usepackage{hyperref}
\usepackage{url}

\usepackage{amsthm}
\usepackage{amsfonts}
\usepackage{amsmath}
\usepackage{amssymb}
\usepackage{mathtools}
\usepackage{bm}
\usepackage{bbm} 
\usepackage{xfrac}
\usepackage{enumitem}
\usepackage{commands}
\usepackage{multirow}
\usepackage{adjustbox}
\usepackage{commands}

\usepackage{listings}

\DeclareMathOperator{\softmax}{softmax}
\DeclareMathOperator{\exponential}{exponential}
\newcommand{\ra}[1]{\renewcommand{\arraystretch}{#1}}

\newcommand{\estimator}[1]{\nabla_{#1}}
\newcommand{\REINFORCE}{\estimator{\mathrm{REINF}}}
\newcommand{\GS}{\estimator{\mathrm{GS}}}
\newcommand{\ST}{\estimator{\mathrm{ST}}}
\newcommand{\STGS}{\estimator{\mathrm{STGS}}}
\newcommand{\GR}{\estimator{\mathrm{GR}}}
\newcommand{\GRMC}[1]{\estimator{\mathrm{GRMC}#1}}

\newcommand{\vspacesectiontop}{\vspace{0pt}}
\newcommand{\vspacesectionbottom}{\vspace{0pt}}
\newcommand{\vspacesubsectiontop}{\vspace{0pt}}
\newcommand{\vspacesubsectionbottom}{\vspace{0pt}}

\newcommand{\vspacefiguretop}{\vspace{0pt}}

\newcommand{\vspacesubfiguretop}{\vspace{0pt}}
\newcommand{\vspacesubfigurebottom}{\vspace{0pt}}

\newcommand{\vspacefigcaptiontop}{\vspace{-2pt}}
\newcommand{\vspacefigcaptionbottom}{\vspace{0pt}}
\newcommand{\vspacesubcaptiontop}{\vspace{-18pt}}
\newcommand{\vspacesubcaptionbottom}{\vspace{0pt}}

\newcommand{\vspacetabletop}{\vspace{0pt}}
\newcommand{\vspacetablebottom}{\vspace{0pt}}
\newcommand{\vspacetablecaptiontop}{\vspace{-2pt}}
\newcommand{\vspacetablecaptionbottom}{\vspace{-8pt}}

\title{Rao-Blackwellizing the Straight-Through Gumbel-Softmax Gradient Estimator}

\author{Max B. Paulus \\
ETH Z\"urich \\
\texttt{max.paulus@inf.ethz.ch} \\
\And
Chris J. Maddison\thanks{Work done partly at the Institute for Advanced Study, Princeton, NJ.} \\
University of Toronto, Vector Institute \\
\texttt{cmaddis@cs.toronto.ca}
\And
Andreas Krause \\ 
ETH Z\"urich \\
\texttt{krausea@ethz.ch} \\
}

\iclrfinalcopy 
\begin{document}

\maketitle

\begin{abstract}
Gradient estimation in models with discrete latent variables is a challenging problem, because the simplest unbiased estimators tend to have high variance. To counteract this, modern estimators either introduce bias, rely on multiple function evaluations, or use learned, input-dependent baselines. Thus, there is a need for estimators that require minimal tuning, are computationally cheap, and have low mean squared error. In this paper, we show that the variance of the straight-through variant of the popular Gumbel-Softmax estimator can be reduced through Rao-Blackwellization without increasing the number of function evaluations. This provably reduces the mean squared error. We empirically demonstrate that this leads to variance reduction, faster convergence, and generally improved performance in two unsupervised latent variable models.
\end{abstract}
\vspacesectiontop
\section{Introduction}
\vspacesectionbottom

\looseness -1 Models with discrete latent variables are common in machine learning. Discrete random variables provide an effective way to parameterize multi-modal distributions, and some domains naturally have latent discrete structure (e.g, parse trees in NLP). Thus, discrete latent variable models can be found across a diverse set of tasks, including conditional density estimation, generative text modelling \citep{yang2017improved}, multi-agent reinforcement learning \citep{mordatch2017emergence, lowe2017multi} or conditional computation \citep{bengio2013straightthrough, davis2013low}. 

\looseness -1 The majority of these models are trained to minimize an expected loss using gradient-based optimization, so the problem of gradient estimation for discrete latent variable models has received considerable attention over recent years. Existing estimation techniques can be broadly categorized into two groups, based on whether they require one loss evaluation \citep{glynn1990likelihood, williams1992simple, bengio2013straightthrough, mnih2014neural, chung2017hierarchical, maddison2017concrete, jang2017categorical, grathwohl2018backpropagation} or multiple loss evaluations \citep{gu2015muprop, mnih2016variational, tucker2017rebar} per estimate. These estimators reduce variance by introducing bias or increasing the computational cost with the overall goal being to reduce the total mean squared error.

\looseness -1 Because loss evaluations are  {costly} in the modern deep learning age,  {single evaluation} estimators are particularly desirable. This family of estimators can be further categorized into those that relax the discrete randomness in the forward pass of the model \citep{maddison2017concrete, jang2017categorical} and those that leave the loss computation unmodified \citep{glynn1990likelihood, williams1992simple, bengio2013straightthrough, chung2017hierarchical, mnih2014neural, grathwohl2018backpropagation}. The ones that do not modify the loss computation are preferred, because they avoid the accumulation of errors in the forward direction and they allow the model to exploit the sparsity of discrete computation. Thus, there is a particular need for {single evaluation estimators that do not modify the loss computation}.

In this paper we introduce such a method. In particular, we propose a {Rao-Blackwellization} scheme for the straight-through variant of the Gumbel-Softmax estimator \citep{jang2017categorical, maddison2017concrete}, which comes at a minimal cost, and does not increase the number of function evaluations. The {\em straight-through Gumbel-Softmax estimator}\citep[ST-GS,][]{jang2017categorical} is a lightweight state-of-the-art single-evaluation estimator based on the Gumbel-Max trick  \citep[see][and references therein]{maddison2014astar}. The ST-GS uses the argmax over Gumbel random variables to generate a discrete random outcome in the forward pass. It computes derivatives via backpropagation through a tempered \textit{softmax} of the same Gumbel sample. Our Rao-Blackwellization scheme is based on the key insight that there are {\em many} configurations of Gumbels corresponding to the {\em same} discrete random outcome and that these can be {marginalized} over with Monte Carlo estimation. By design, there is no need to re-evaluate the loss and the additional cost of our estimator is linear only in the number of Gumbels needed for a single forward pass. As we show, the Rao-Blackwell theorem implies that our estimator has lower mean squared error than the vanilla ST-GS. We demonstrate the effectiveness of our estimator in unsupervised parsing on the ListOps dataset \citep{nangia2018listops} and on a variational autoencoder loss \citep{kingma2013vae, rezende2014stochastic}. We find that in practice our estimator trains {\em faster} and achieves {\em better test set performance}. The magnitude of the improvement depends on several factors, but is particularly pronounced at small batch sizes and low temperatures.

\vspacesectiontop
\section{Background}
\vspacesectionbottom

For clarity, we consider the following simplified scenario. Let $D \sim p_{\theta}$ be a discrete random variable $D \in \{0, 1\}^n$ in a one-hot encoding, $\sum D_i = 1$, with distribution given by  $p_{\theta}(D) \propto \exp(D^T \theta)$ where $\theta \in \R^n$. Given a continuously differentiable $f : \R^{2n} \to \R$, we wish to minimize,
\begin{align}
    \label{eq:problem} \min_{\theta} \mathbb{E}[f(D, \theta)],
\end{align}
where the expectation is taken over all of the randomness. In general $\theta$ may be computed with some neural network, so our aim is to derive estimators of the total derivative of the expectation with respect to $\theta$ for use in stochastic gradient descent. This framework covers most simple discrete latent variable models, including variational autoencoders \citep{kingma2013vae, rezende2014stochastic}. 

The {\em REINFORCE estimator} \citep{glynn1990likelihood, williams1992simple} is { unbiased} (under certain smoothness assumptions) and given by:
\begin{align}
\label{eq:def_reinforce}
    \REINFORCE := f(D, \theta) \frac{\partial  \log p_{\theta}(D)}{\partial \theta} + \frac{\partial f(D, \theta)}{\partial \theta} .
\end{align}
\looseness -1 Without careful use of control variates \citep{mnih2014neural, tucker2017rebar, grathwohl2018backpropagation}, the REINFORCE estimator tends to have { prohibitively high variance}. To simplify exposition we assume henceforth that $f(D, \theta) = f(D)$ does not depend on $\theta$, because the dependence of $f(D, \theta)$ on $\theta$ is accounted for in the second term of \eqref{eq:def_reinforce}, which is shared by most estimators and generally has low variance.

One strategy for reducing the variance is to introduce {bias} through a relaxation \citep{jang2017categorical, maddison2017concrete}. Define the tempered softmax $\softmax_{\tau} : \R^n \to \R^n$ by $\softmax_{\tau}(x)_i = \exp(x_i/\tau) / \sum_{j=1}^n \exp(x_j/\tau)$. The relaxations are based on the observation that the sampling of $D$ can be reparameterized using Gumbel random variables and the zero-temperature limit of the tempered softmax under the coupling:
\begin{align}
\label{eq:coupling} D = \lim_{\tau \to 0} S_{\tau}; \qquad S_{\tau} = \softmax_{\tau}(\theta + G)
\end{align}
where $G$ is a vector of i.i.d. $G_i \sim \Gumbel$ random variables. At finite temperatures $S_{\tau}$ is known as a Gumbel-Softmax (GS) \citep{jang2017categorical} or concrete \citep{maddison2017concrete} random variable, and the relaxed loss $\mathbb{E}[f(S_{\tau}, \theta)]$ admits the following reparameterization gradient estimator for $\tau > 0$:\footnote{For a function $f(x_1, x_2)$, $\partial f(z_1, z_2) / \partial x_1$ is the partial derivative (e.g., a gradient vector) of $f$ in the first variable evaluated at $z_1, z_2$. For a function $g(\theta)$, $d g / d \theta$ is the total derivative of $g$ in $\theta$. For example, $ d \softmax_{\tau}(\theta + G)/ d\theta$ is the Jacobian of the tempered softmax evaluated at the random variable $\theta + G$.}
\begin{align}
    \GS := \frac{\partial  f(S_{\tau})}{\partial S_{\tau}} \frac{d  \softmax_{\tau}(\theta + G)}{d \theta}.
\end{align}
This is an unbiased estimator of the gradient of  $\mathbb{E}[f(S_{\tau}, \theta)]$, but a biased estimator of our original problem (\ref{eq:problem}). For this to be well-defined $f$ must be defined on the interior of the simplex (where $S_{\tau}$ sits). This estimator has the advantage that it is easy to implement and generally low-variance, but the disadvantage that it modifies the forward computation of $f$ and is biased.  Henceforth, we assume $D, S_{\tau},$ and $G$ are coupled almost surely through (\ref{eq:coupling}).

Another popular family of estimators are the so-called {\em straight-through estimators} \citep[c.f.,][]{bengio2013straightthrough, chung2017hierarchical}. In this family, the forward computation of $f$ is unchanged, but backpropagation is computed ``through'' a surrogate. One popular variant takes as a surrogate the tempered probabilities of $D$, resulting in the \textit{slope-annealed straight-through estimator (ST)}:
\begin{align}
    \ST := \frac{\partial  f(D)}{\partial D} \frac{d \softmax_{\tau}(\theta)}{d \theta}.
\end{align}
The most popular variant \citep{jang2017categorical} is known as the {\em straight-through Gumbel-Softmax (ST-GS)}. The surrogate for ST-GS is $S_{\tau}$, whose Gumbels are coupled to $D$ through \eqref{eq:coupling}:
\begin{align}
\label{eq:def_stgs}
    \STGS := \frac{\partial  f(D)}{\partial D} \frac{d \softmax_{\tau}(\theta + G)}{d \theta}.
\end{align}
The straight-through family has the advantage that they tend to be low-variance and $f$ need not be defined on the interior of the simplex (although $f$ must be differentiable at the corners). This family has the disadvantage that they are not known to be unbiased estimators of \emph{any} gradient. These estimators are quite popular in practice, because they preserve the forward computation of $f$, which prevents the forward propagation of errors and maintains sparsity \citep{choi2017unsupervised, chung2017hierarchical, bengio2013straightthrough}.

All of the estimators discussed in this paper can be computed by any of the standard automatic differentiation software packages using a single evaluation of $f$ on a realization of $D$ or some underlying randomness. We present implementation details for these and our Gumbel-Rao estimator in the Appendix, emphasizing the surrogate loss framework \citep{schulman2015gradient, weber2019credit} and considering the multiple stochastic layer case not covered by \eqref{eq:problem}.

\vspacesectiontop
\section{Gumbel-Rao Gradient Estimator}
\vspacesectionbottom

\vspacesubsectiontop
\subsection{Rao-Blackwellization of ST-Gumbel-Softmax}
\vspacesubsectionbottom

We now derive our Rao-Blackwelization scheme for the ST-GS estimator. Our approach is based on the observation that there is a {\em many-to-one} relationship between realizations of $\theta + G$ and $D$ in the coupling described by \eqref{eq:coupling} and that the variance introduced by $\theta + G$ can be marginalized out. The resulting estimator, which we call the {\em Gumbel-Rao (GR)} estimator, is guaranteed by the Rao-Blackwell theorem to have lower variance than ST-GS. In the next subsection we turn to the practical question of carrying out this marginalization.

In the Gumbel-max trick \eqref{eq:coupling}, $D$ is a one-hot indicator of the index of $\arg\max_{i} \left\{ \theta_i + G_i\right\}$. Because this argmax operation is non-invertible, there are { many} configurations of $\theta + G$ that correspond to a { single} $D$ outcome. Consider an alternate factorization of the joint distribution of $(\theta + G, D)$: first sample $D \sim p_{\theta}$, and then $\theta + G$ given $D$. In this view, the Gumbels are auxillary random variables, at which the Jacobian of the tempered softmax is evaluated and which locally increase the variance of the estimator. This local variance can be removed by marginalization. This is the key insight of our GR estimator, which is given by,
\begin{align}
\label{eq:def_stgr}
    \GR := \frac{\partial  f(D)}{\partial D} \mathbb{E}\left[\frac{d \softmax_{\tau}(\theta + G)}{d \theta} \middle | D \right].
\end{align}
It is not too difficult to see that $\GR = \expect\left[\STGS | D \right]$. By the tower rule of expectation, GR has the same expected value as ST-GS and is an instance of a {Rao-Blackwell estimator} \citep{blackwell1947conditional, Rao1992information}. Thus, it has the same mean as ST-GS, but a lower variance. Taken together, these facts imply that GR enjoys a lower mean squared error (\emph{not} a lower bias) than ST-GS. 
\begin{prop}
\label{prop:gr_mse}
Let $\STGS$ and $\GR$ be the estimators defined in \eqref{eq:def_stgs} and \eqref{eq:def_stgr}. Let $\nabla_{\theta} := d \mathbb{E}[f(D)]/d\theta$ be the true gradient that we are trying to estimate. We have
\begin{align}
\label{eq:gr_mse}
\expect
\left[
\left\lVert \GR - \nabla_{\theta} \right\rVert^2
\right] 
\leq 
\expect
\left[
\left\lVert \STGS - \nabla_{\theta} \right\rVert^2
\right].
\end{align}
\end{prop}
\begin{proof}
The proposition follows from Jensen's inequality and the linearity of expectations, see \ref{appendix:proofs_gr_mse}.
\end{proof}

While GR is only guaranteed to reduce the variance of ST-GS, Proposition \ref{prop:gr_mse} guarantees that, as a function of $\tau$, the MSE of GR is a pointwise lower bound on ST-GS. This means GR can be used for estimation at lower temperatures, where ST-GS has high variance and low bias. Empirically, we observe that our estimator indeed facilitates training at lower temperatures and thus results in an estimator that improves in both bias and variance over ST-GS. Thus, this estimator retains the favourable properties of the ST-GS (single, unmodified evaluation of $f$) while improving its performance.

\vspacesubsectiontop
\subsection{Monte Carlo Approximation}
\vspacesubsectionbottom

The GR estimator requires computing the expected value of the Jacobian of the tempered softmax over the distribution $\theta + G | D$. Unfortunately, an analytical expression for this is only available in the simplest cases.\footnote{For example, in the case of $n=2$ (binary) and $\tau=1$ an analytical expression for the GR estimator is available.} In this section we provide a simple Monte Carlo (MC) estimator with sample size $K$ for $\mathbb{E}[d S_{\tau} / d \theta | D]$, which we call the \textit{Gumbel-Rao Monte Carlo Estimator (GR-MC$K$)}. This estimator can be computed locally at a cost that only scales like $nK$ (the arity of $D$ times $K$).

They key property exploited by GR-MC$K$ is that $\theta + G | D$ can be reparameterized in the following closed form. Given a realization of $D$ such that $D_i = 1$, $Z(\theta) = \sum_{i=1}^n \exp(\theta_i)$, and $E_j \sim \exponential$ i.i.d., we have the following equivalence in distribution \citep{maddison2014astar, maddison2016ppmontecarlo, tucker2017rebar}.
\begin{align}
\label{eq:gumbel_reparam}
	\theta_j + G_j | D \overset{d}{=} 
	\begin{cases}
-\log\left(E_j\right) + \log Z(\theta) & \text{if } j = i \\
-\log\left(\frac{E_j}{\exp(\theta_j)} + \frac{E_i}{Z(\theta)}\right) & \text{o.w.}
\end{cases}
\end{align}
With this in mind, we define the GR-MC$K$ estimator:
\begin{align}
\label{eq:def_grmc}
\GRMC{K} := \frac{\partial  f(D)}{\partial D} \left[ \frac{1}{K}\sum_{k=1}^K \frac{d \softmax_{\tau}(\theta + G^k)}{d \theta}\right],
\end{align}
where $G^k \sim \theta + G | D$ i.i.d. using the reparameterization \eqref{eq:gumbel_reparam}. Note that the total derivative $d \softmax_{\tau}(\theta + G^k)/d \theta$ is taken through both $\theta$ and $G^k$. For the case $K=1$, our estimator reduces to the standard ST-GS estimator. The cost for drawing multiple samples $G^k \sim \theta + G | D$ scales only {\em linearly} in the arity of $D$ and is usually negligible in modern applications, where the bulk of computation accrues from the computation of $f$. Moreover, drawing multiple samples of $\theta + G | D$ can easily be parallelised on modern workstations (GPUs, etc.). Our estimator remains a single-evaluation estimator under this scheme, because the loss function $f$ is still only evaluated at $D$. Finally, as with GR, the GR-MC$K$ is guaranteed to improve in MSE over ST-GS for any $K \geq 1$, as confirmed in Proposition \ref{prop:grmc_mse}.

\begin{prop}
\label{prop:grmc_mse}
Let $\STGS$ and $\GRMC{K}$ be the estimators defined in \eqref{eq:def_stgs} and \eqref{eq:def_grmc}. Let $\nabla_{\theta} := d \mathbb{E}[f(D)]/d\theta$ be the true gradient that we are trying to estimate. For all $K \geq 1$, we have
\begin{align}
\label{eq:grmc_mse}
\expect
\left[
\left\lVert \GRMC{K} - \nabla_{\theta} \right\rVert^2
\right] 
\leq 
\expect
\left[
\left\lVert \STGS - \nabla_{\theta} \right\rVert^2
\right].
\end{align}
\end{prop}
\begin{proof}
	The proposition follows from Jensen's inequality and the linearity of expectations, see \ref{appendix:proofs_grmc_mse}.
\end{proof}

\subsection{Variance Reduction in Minibatches}
\label{gumbel_rao:mb}
The variance of GR-MC$K$ can be reduced by increasing $K$ or by averaging $B$ i.i.d. samples of the GR-MC$K$ estimator. An average of i.i.d. samples $\GRMC{K}^b$ for $b \in \{1, \ldots, B\}$ is an generalization of minibatching by sampling data points with replacement. In this subsection, we consider the effect of increasing $K$ and $B$ separately. 

\newcommand{\GRMCMB}[2]{\overline{\nabla}_{\mathrm{GRMC}#1}^{1:#2}}

Let $\GRMC{K}^b$ be i.i.d. as $\GRMC{K}$ for $b \in \{1, \ldots, B\}$ and define the following ``minibatched'' GR-MC$K$ estimator:
\begin{align}
\label{eq:grmck_minibatched}
\GRMCMB{K}{B} := \frac{1}{B}\sum_{b=1}^B\GRMC{K}^b.
\end{align}
Proposition \ref{prop:grmc_var} summarizes the scaling of the variance of \eqref{eq:grmck_minibatched}, and is an elementary application of the law of total variance.
\begin{prop}
\label{prop:grmc_var}
Let $\STGS$, $\GR$ and $\GRMCMB{K}{B}$ be the estimators defined in \eqref{eq:def_stgs}, \eqref{eq:def_stgr} and \eqref{eq:grmck_minibatched}. We have
\begin{align}
\label{eq:grmc_var}
\var\left[\GRMCMB{K}{B}\right]
&= 
\frac{\expect
\left[
\var
\left[
\STGS | D 
\right]
\right]}{BK}
+
\frac{\var
\left[
\GR
\right]}{B}
\end{align}
where $\var$ is the trace of the covariance matrix.
\end{prop}
\begin{proof}
	The proposition follows directly from the law of total variance, see \ref{appendix:proofs_grmc_var}.
\end{proof}
As expected the total variance of $\GRMCMB{K}{B}$ decreases like $1/B$. The key point of Proposition \ref{prop:grmc_var} is that the component of the variance that $K$ reduces can also be reduced by increasing the batch size $B$. This suggests that the effect of GR-MC$K$ will be most pronounced at small batch sizes. Proposition \ref{prop:grmc_var} also indicates that there are diminishing returns to increasing $K$ for a fixed batch size $B$, such that the variance of GR-MC$K$ will eventually be dominated by the right-hand term of \eqref{eq:grmc_var}. In our experimental section, we explore various $K$ and study the effect on gradient estimation in more detail. 

\looseness -1 Finally, we note that the choice of a Monte Carlo scheme to approximate $\expect\left[d S_{\tau}/d \theta|D\right]$ permits the use of additional well-known variance reduction methods to improve the estimation properties of our gradient estimator. For example, antithetic variates or importance sampling are sensible methods to explore in this setting \citep{kroese2013handbook}. For low-dimensional discrete random variables, Gaussian quadrature or other numerical methods could be employed. However, we found the simple Monte Carlo scheme described above effective in practice and report results based on this procedure in the experimental section.

\vspacesectiontop
\section{Related Work}
\vspacesectionbottom
The idea of using Rao-Blackwellization to reduce the variance of gradient estimators for discrete latent variable models has been explored in machine learning. For example, \citet{liu2018rao} describe a sum-and-sample style estimator that analytically computes part of the expectation to reduce the variance of the gradient estimates. The favorable properties of their estimator are due to the Rao-Blackwell theorem. \citet{Kool2020Estimating} describe a gradient estimator based on sampling without replacement. Their estimator emerges naturally as the Rao-Blackwell estimator of  the importance-weighted estimator \citep{vieira2017estimating} and the estimator described by \citet{liu2018rao}. Both of these estimators rely on {\em multiple} function evaluations to compute a gradient estimate. In contrast, our work is the first to consider Rao-Blackwellisation in the context of a {\em single-evaluation} estimator. 
\vspacesectiontop
\section{Experiments}
\label{sec:exp}
\vspacesectionbottom

\subsection{Protocol}
In this section, we study the effectiveness of our gradient estimator in practice. In particular, we evaluate its performance with respect to the temperature $\tau$, the number of MC samples $K$ and the batch size $B$. We measure the variance reduction and improvements in MSE our estimator achieves in practice, and assess whether its lower variance gradient estimates accelerate the convergence on the objective or improve final test set performance. Our focus is on single-evaluation gradient estimation and we compare against other non-relaxing estimators (ST, ST-GS and REINFORCE with a running mean as a baseline) and relaxing estimators (GS), where permissible. Experimental details are given in Appendix \ref{appendix:exp_details}.

First, we consider a toy example which allows us to explore and visualize the variance of our estimator and suggests that it is particularly effective at low temperatures. Next, we evaluate the effect of $\tau$ and $K$ in a latent parse tree task which does not permit the use of relaxed gradient estimators. Here, our estimator facilitates training at low temperatures to improve overall performance and is effective even with few MC samples. Finally, we train variational auto-encoders with discrete latent variables \citep{kingma2013vae, rezende2014stochastic}. Our estimator yields improvements at small batch sizes and obtains competitive or better performance than the GS estimator at the largest arity.

\subsection{Quadratic Programming on the Simplex}

\vspacefiguretop
\begin{figure}[t]
	\centering
	\vspacesubfiguretop
	\begin{subfigure}[t]{0.22\textwidth}
		\centering
		\includegraphics[width=\textwidth]{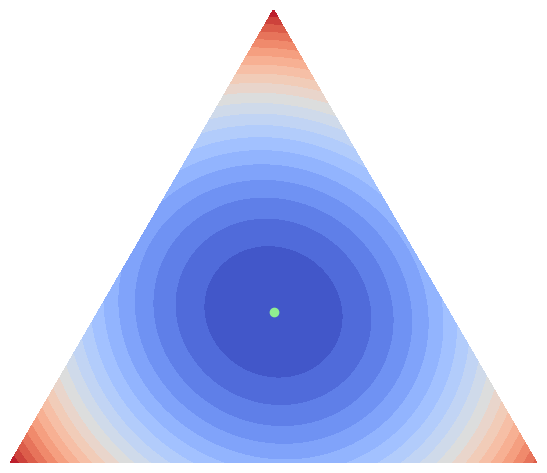}
		\vspacesubcaptiontop
		\caption{Objective function}
		\vspacesubcaptionbottom
		\label{fig:toy:qp}
	\end{subfigure}
	\vspacesubfigurebottom
	\hfill
	\vspacesubfiguretop
	\begin{subfigure}[t]{0.22\textwidth}
		\centering
		\includegraphics[width=\textwidth]{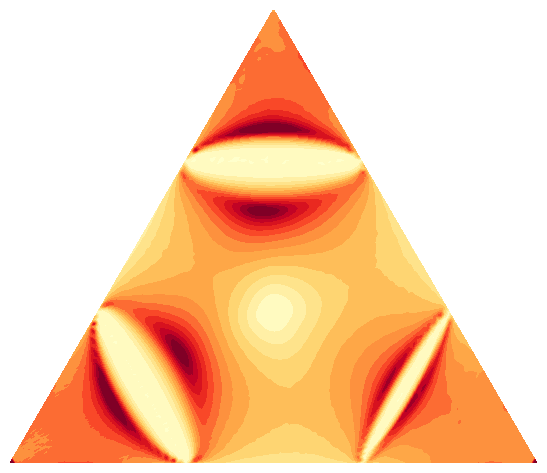}
		\vspacesubcaptiontop		
		\caption{$\Delta \log_{10}(\var)$ ($\tau$=0.1)}	
		\vspacesubcaptionbottom
		\label{fig:toy:tau01}		
	\end{subfigure}
	\vspacesubfigurebottom	
	\hfill
	\vspacesubfiguretop
	\begin{subfigure}[t]{0.22\textwidth}
		\centering
		\includegraphics[width=\textwidth]{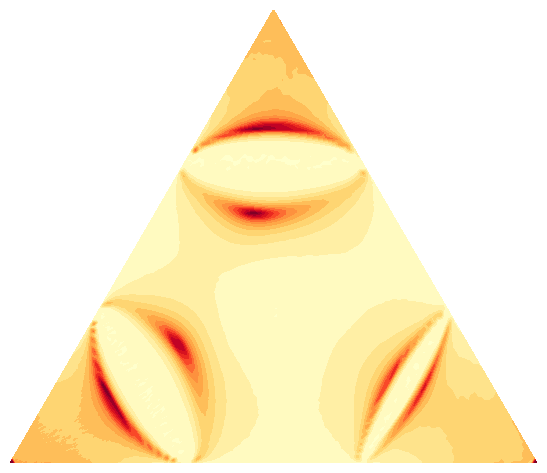}
		\vspacesubcaptiontop
		\caption{$\Delta \log_{10}(\var)$ ($\tau$=0.5)}		
		\vspacesubcaptionbottom		
		\label{fig:toy:tau05}				
	\end{subfigure}	
	\vspacesubfigurebottom
	\hfill
	\vspacesubfiguretop
	\begin{subfigure}[t]{0.22\textwidth}
		\centering
		\includegraphics[width=\textwidth]{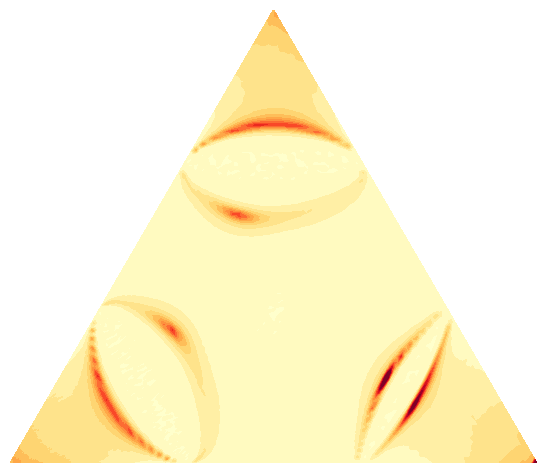}
		\vspacesubcaptiontop
		\caption{$\Delta \log_{10}(\var)$ ($\tau$=1.0)}			
		\vspacesubcaptionbottom
		\label{fig:toy:tau1}				
	\end{subfigure}
	\vspacesubfigurebottom	
	\hfill
	\vspacesubfiguretop
	\begin{subfigure}[t]{0.04\textwidth}
		\centering
		\includegraphics[scale=0.26]{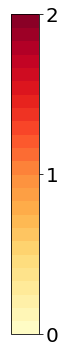}
	\end{subfigure}	
	\vspacesubfigurebottom	
	\vspacefigcaptiontop
	\caption{Our estimator (GR-MC$K$) effectively reduces the variance over the entire simplex and is particularly effective at low temperatures. Contours for the quadratic programme in three dimensions (\ref{fig:toy:qp}) and  difference in log10-trace of the covariance  matrix between ST-GS and GR-MC$1000$ at different temperatures (\ref{fig:toy:tau01},  \ref{fig:toy:tau05},  \ref{fig:toy:tau1}). Warmer means difference is larger.}
	\vspacefigcaptionbottom
	\label{fig:toy}
\end{figure}

As a toy problem, we consider the problem of minimizing a quadratic program $(p-c)^\intercal Q (p-c)$ over the probability simplex $\Delta^{n-1} = \{ p \in R^n : p_i \geq 0, \sum_{i=1}^n p_i = 1\}$ for $Q \in R^{n \times n}$ positive-definite and $c	\in R^n$. This problem may be reframed as the following stochastic optimization problem,
\begin{align*}
	\min_{p \in \Delta^{n-1}} \mathbb{E}[ (D-c)^\intercal A(p) (D-c) ],
\end{align*}
where $D \sim \text{Discrete}(p)$ and $A_{ii}(p)=\frac{(p_i - c_i)^2}{p_i - 2 p_i c_i + c_i^2} Q_{ii}$ and $A_{ij}(p)=\frac{(p_i - c_i)(p_j - c_j)}{c_i c_j - p_i c_j - c_i p_j} Q_{ij}$ for $i \neq j$. 
While solving the above problem is simple using standard methods, it provides a useful testbed to evaluate the effectiveness of our variance reduction scheme. For this purpose, we consider $Q_{ij} = \exp\left(-2|i-j|\right)$  and $c_i = \frac{1}{3}$ in three dimensions. 

Our estimator reduces the variance in the gradient estimation over the entire simplex and is particularly effective at low temperatures in this problem. In Figure \ref{fig:toy}, we compare the log10-trace of the covariance matrix of ST-GS and GR-MC1000 at three different temperatures and display their difference over the entire domain. The improvement is universal. The pattern is not always intuitive (oval bull's eyes), despite the simplicity of the objective function. Compared with ST-GS, our estimator on this example appears more effective closer to the corners and edges, which is important for learning discrete distributions. At lower temperatures, the difference between the two estimators becomes particularly acute. This suggests that our estimator may train better at lower temperatures and be more responsive to optimizing over the temperature to successfully trade off bias and variance.

\subsection{Unsupervised Parsing on ListOps}

Straight-through estimators feature prominently in NLP \citep{martins-etal-2019-latent} where latent discrete structure arises naturally, but the use of relaxations is often infeasible. Therefore, we evaluate our estimator in a latent parse tree task on subsets of the ListOps dataset \citep{nangia2018listops}. This dataset contains sequences of prefix arithmetic expressions $x$ (e.g., \texttt{max[ 3 min[ 8 2 ]]}) that evaluate to an integer $y \in \{0, 1, \ldots 9\}$. The arithmetic syntax induces a latent parse tree $T$. We consider the model by \citep{choi2017unsupervised} that learns a distribution over plausible parse trees of a given sequence to maximize 
\begin{align*}
	\expect_{q_{\theta}(T|x)}
	\left[
	\log 
	p_{\phi}(y|T, x)
	\right].
\end{align*}
Both the conditional distribution over parse trees $q_{\theta}(T|x)$ and the classifier $p_{\phi}(y|T, x)$ are parameterized using neural networks. In this model, a parse tree $T \sim q_{\theta}(T | x)$ for a given sentence is sampled bottom-up by successively combining the embeddings of two tokens that appear in a given sequence until a single embedding for the entire sequence remains. This is then used for performing the subsequent classification. Because it is computationally infeasible to marginalize over all trees, \citet{choi2017unsupervised} rely on the ST-GS estimator for training. We compare this estimator against our estimator GR-MC$K$ with $K \in \{10, 100, 1000\}$. We consider temperatures $\tau \in \{0.01, 0.1, 1.0\}$ and experiment with shallow and deeper trees by considering sequences of length $L$ up to 10, 25 and 50. All models are trained with stochastic gradient descent with a batch size equal to the maximum $L$. Details are in Appendix \ref{appendix:exp_nlp}.

Our estimator facilitates training at lower temperatures and achieves  better final test set accuracy than ST-GS (Table \ref{table:listops:results}). Increasing $K$ improves the performance at low temperatures, where the differences between the estimators are most pronounced. Overall, across all temperatures this results in modest improvements, particularly for shallow trees and small batch sizes. We also find evidence for diminishing returns: The differences between ST-GS and GR-MC$10$ are larger than between GR-MC$100$ or GR-MC$1000$, suggesting that our estimator is effective even with few MC samples.

\begin{table*}[t]
\vspacetabletop
\ra{1.2}
\vspacetablecaptiontop
\caption{Our estimator (GR-MC$K$) facilitates training at lower temperatures with improved performance on the latent parse tree task. Best test classification accuracy on the ListOps dataset selected on the validation set. Best estimator at given temperature in bold, best estimator across temperatures in italics. Higher is better.}
\vspacetablecaptionbottom
\label{table:listops:results}
\begin{center}
\begin{adjustbox}{width=1\textwidth}
\begin{small}	
\begin{sc}
\begin{tabular}{@{}lcrrrrrrrrr@{}}
\toprule
	&& \multicolumn{3}{c}{$L \leq 10$} &  \multicolumn{3}{c}{$L \leq 25$} &  \multicolumn{3}{c}{$L \leq 50$} \\
\cmidrule(lr){3-5} \cmidrule(lr){6-8} \cmidrule(lr){9-11} 
Estimator && $\tau=0.01$ & $\tau=0.1$ &  $\tau=1.0$ & $\tau=0.01$ &  $\tau=0.1$ & $\tau=1.0$ & $\tau=0.01$ & $\tau=0.1$ & $\tau=1.0$ \\
\midrule
ST-GS && 38.8 & 59.3 & 65.8 & 41.2 & 57.1 & 60.2 & 46.8 & 56.8 & 59.6 \\
GR-MC10 &&  66.4 & 66.9 & 66.7 & 60.7 & 60.8 & 60.9 & 58.7  & 59.1 & 59.6 \\
GR-MC100 && {65.6} & {{66.3}} & {65.9} & 60.0 & \emph{\textbf{61.3}} & \textbf{61.2} & 59.6 & 59.1  & {59.6} \\
GR-MC1000 &&  \textbf{66.5} & \emph{\textbf{67.1}} & \textbf{67.0} & \textbf{60.2} & 60.9 & \textbf{61.2} & \emph{\textbf{60.0}} & \textbf{59.8} & \textbf{59.9} \\
\bottomrule
\end{tabular}
\end{sc}
\end{small}
\end{adjustbox}
\end{center}
\vspacetablebottom
\end{table*}

\vspacefiguretop
\begin{figure*}[b]
	\centering
	\vspacesubfiguretop	
	\begin{subfigure}[t]{0.32\textwidth}
		\centering
		\includegraphics[width=\textwidth]{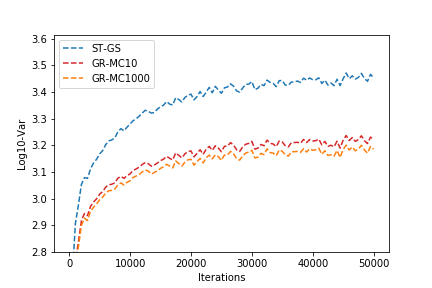}
		\vspacesubcaptiontop				
		\caption{$\log_{10}(\var)$ vs Iterations}
		\vspacesubcaptionbottom		
		\label{fig:vae_varmse:var}
	\end{subfigure}
	\vspacesubfigurebottom	
	\hfill
	\vspacesubfiguretop	
	\begin{subfigure}[t]{0.32\textwidth}
		\centering
		\includegraphics[width=\textwidth]{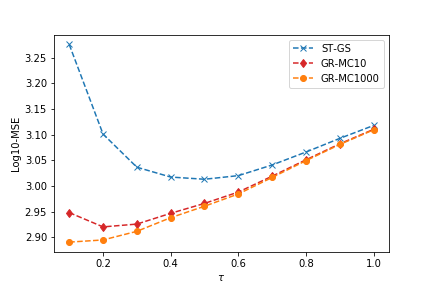}
		\vspacesubcaptiontop				
		\caption{$\log_{10}(\mse)$ vs $\tau$}
		\vspacesubcaptionbottom		
		\label{fig:vae_varmse:mse}		
	\end{subfigure}
	\vspacesubfigurebottom	
	\hfill
	\vspacesubfiguretop	
	\begin{subfigure}[t]{0.32\textwidth}
		\centering
		\includegraphics[width=\textwidth]{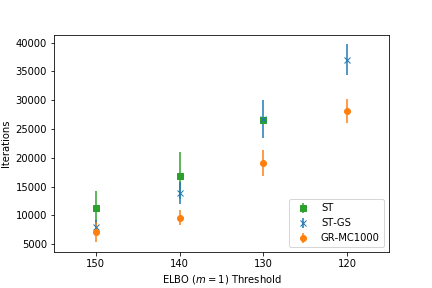}
		\vspacesubcaptiontop				
		\caption{Iterations vs ELBO}
		\vspacesubcaptionbottom				
		\label{fig:vae_varconv_n16}
	\end{subfigure}	
	\vspacesubfigurebottom	
	\label{fig:vae_varmse}
	\vspacefigcaptiontop	
	\caption{Our estimator (GR-MC$K$) effectively reduces the variance over the entire training trajectory (\ref{fig:vae_varmse:var}), achieves a lower mean squared error at a lower temperature (\ref{fig:vae_varmse:mse}) and converges faster than ST and ST-GS on the discrete VAE objective (\ref{fig:vae_varconv_n16}).  Log10-trace of the covariance matrix over a training trajectory (\ref{fig:vae_varmse:var})  and log10-MSE (\ref{fig:vae_varmse:mse}) at different temperatures during training, average number of iterations and standard error to reach various thresholds of the objective on the validation set (\ref{fig:vae_varconv_n16}).}
	\vspacefigcaptionbottom
\end{figure*}
\vspacefiguretop

\subsection{Generative Modeling with Discrete Variational Auto-Encoders}
Finally, we train variational auto-encoders \citep{kingma2013vae, rezende2014stochastic} with discrete latent random variables on the MNIST dataset of handwritten digits \citep{lecun2010mnist}. We used the fixed binarization of \citep{salakhutdinov2008quantitative} and the standard split into train, validation and test sets. Our objective is to maximize the following variational lower bound on the log-likelihood,
\begin{align*}
	\log p(x) > 
	\expect_{q_{\theta}(D^i|x)}
	\left[
	\log 
	\left(
	\frac{1}{M} 
	\sum_{j=1}^{M} 
	\frac{p_{\phi}(x, D^i)}{ q_{\theta}(D^j|x)}
	\right)
	\right]
\end{align*}
where $x$ denotes the input image and $D^i \sim q_{\theta}(D^i | x)$ denotes a vector of discrete latent random variables. This objective takes a form in equation (\ref{eq:problem}). For training, the bound is approximated using only a single sample ($M=1$). For final validation and testing, we use 5000 samples ($M=5000$). Both the generative model $p_{\phi}(x, D)$ and the variational distributions $q_{\theta}(D|x)$ were parameterized using neural networks. We experiment with different batch sizes and discrete random variables of arities in $\{2, 4, 8, 16\}$ as in \citet{maddison2017concrete}. To facilitate comparisons, we do not alter the total dimension of the latent space and train all models for 50,000 iterations using stochastic gradient descent with momentum. Hyperparameters are optimised for each estimator using random search \citep{bergstra2012random} over twenty independent runs. More details are given in Appendix \ref{appendix:exp_vae}.

Our estimator effectively reduces the variance over the entire training trajectory (Figure \ref{fig:vae_varmse:var}). Even a small number of MC samples ($K=10$) results in sizable variance reductions. The variance reduction compares favorably to the magnitude of the mini-batch variance (Appendix \ref{appendix:add_fig}). As a result, our estimator facilitates training at lower temperatures and features a lower MSE (Figure \ref{fig:vae_varmse:mse}). During training our estimator can trade off bias and variance to improve the gradient estimation. Empirically, we observed that on this task, the best models using ST-GS trained at an average temperature of $0.65$, while the best models using GR-MC1000 trained at an average temperature of $0.35$. This is interesting, because it indicates that our estimator may make the use of temperature annealing during training more effective. We find lower variance gradient estimates improve convergence of the objective (Figure \ref{fig:vae_varconv_n16}). GR-MC1000 reaches various performance thresholds on the validation set with reliably fewer iterations than ST or ST-GS. This effect is observable at different arities and persistent over the entire training trajectory.

	For final test set performance, our estimator outperforms ST and REINFORCE (Table \ref{table:vae:results}). The improvements over ST-GS extend up to two nats (for batch size 20, 16-ary) at small batch sizes and are more modest at large batch sizes as expected (also see Appendix \ref{appendix:add_fig}). This confirms that our estimator might be particularly effective in settings, where training at high batch sizes is prohibitively expensive. The improvements from increasing the number of MC samples tend to saturate at $K=100$ on this task. Further, our results suggest that relaxed estimators may be preferred (if they can be used) for discrete random variables of smaller arity. For example, the GS estimator outperforms all straight-through estimators for binary variables for both batch sizes. For large arities however, we find that straight-through estimators can perform competitively: Our estimator GR-MC1000 achieves the best performance overall and outperforms the GS estimator for 16-ary variables. 
	
\begin{table*}[t]
\vspacetabletop
\ra{1.2}
\vspacetablecaptiontop
\caption{Our estimator (GR-MC$K$) outperforms other straight-through estimators for discrete-latent-space VAE objectives on the MNIST dataset and is competitive with the Gumbel-Softmax ($GS$) at large arities. Best bound on the test negative log-likelihood selected on the validation set. Best straight-through estimator in bold, best estimator in italics. Lower is better.}
\vspacetablecaptionbottom
\label{table:vae:results}
\begin{center}
\begin{adjustbox}{width=1\textwidth}
\begin{small}	
\begin{sc}
\begin{tabular}{@{}lcrrrrrrrr@{}}
\toprule
	&& \multicolumn{2}{c}{binary} &  \multicolumn{2}{c}{$4$-ary} &  \multicolumn{2}{c}{$8$-ary} &  \multicolumn{2}{c}{$16$-ary}\\
\cmidrule(lr){3-4} \cmidrule(lr){5-6} \cmidrule(lr){7-8} \cmidrule(lr){9-10}
Estimator && $B=20$ & $B=200$ &  $B=20$ & $B=200$ &  $B=20$ & $B=200$ &  $B=20$ & $B=200$ \\
\midrule
{GS} &&  \emph{{98.2}} &  \emph{{96.4}} & \emph{{95.7}} &  {93.8} & \emph{{95.5}} &  {92.3} & \emph{{96.8}} &  {94.3} \\
\midrule
REINFORCE &&  202.6 & 121.4 & 173.7 & 122.2 & 203.9 & 124.9 & 169.4 & 129.5 \\
ST &&  105.5 & 103.1 & 106.2 & 104.5 & 107.2 & 105.1 & 108.2 & 104.5 \\
ST-GS &&  100.7 & 97.1 & 99.1 & 93.7 & 98.0 & 92.8& 98.8 & 92.6 \\
GR-MC10 &&  100.7 & 97.4 & 97.8 & 93.8 & 97.4 & 93.1 & 97.9 & 92.4 \\
GR-MC100 && 100.6 & \textbf{96.8} & \textbf{97.5} & 94.0	 & 96.8 & \emph{\textbf{92.2}} & 97.3  & 92.4 \\
GR-MC1000 && \textbf{100.5} & 97.0 & 97.6 & \emph{\textbf{93.5}} & \textbf{96.5} & 92.5 & \emph{\textbf{96.8}} & \emph{\textbf{92.2}} \\
\bottomrule
\end{tabular}
\end{sc}
\end{small}
\end{adjustbox}
\end{center}
\vspacetablebottom
\end{table*}
\vspacesectiontop
\section{Conclusion}
\vspacesectionbottom
We introduced the Gumbel-Rao estimator, a new single-evaluation non-relaxing gradient estimator for models with discrete random variables. Our estimator is a Rao-Blackwellization of the state-of-the-art straight-through Gumbel-Softmax estimator. It enjoys lower variance and can be implemented efficiently using Monte Carlo methods. In particular and in contrast to most other work, it does not require additional function evaluations. Empirically, our estimator improved final test set performance in an unsupervised parsing task and on a variational auto-encoder loss. It accelerated convergence on the objective and compared favorably to other standard gradient estimators. Even though the gains were sometimes modest, they were persistent and particularly pronounced when models were trained at low temperatures or with small batch sizes. We expect that our estimator will be most effective in such settings and that further gains may be uncovered when combining our Rao-Blackwellisation scheme with an annealing schedule for the temperature. Finally, we hope that our work inspires further exploration of the use of Rao-Blackwellisation for gradient estimation. 
\section*{Acknowledgements}
MBP gratefully acknowledges support from the Max Planck ETH Center for Learning Systems. CJM is grateful for the support of the James D. Wolfensohn Fund at the Institute of Advanced Studies in Princeton, NJ. Resources used in preparing this research were provided, in part, by the Sustainable Chemical Processes through Catalysis (Suchcat) National Center of Competence in Research (NCCR), the Province of Ontario, the Government of Canada through CIFAR, and companies sponsoring the Vector Institute.

\bibliography{refs}
\bibliographystyle{iclr2021_conference}

\appendix
\newpage
\appendix

\section{Implementing Gradient Estimators by Modifying Backpropagation}

An advantage of the GRMC-K estimator is the ease with which it can be implemented using automatic differentiation software. Here, we provide a pseudo code template for such an implementation.

\begin{lstlisting}
class GRMCK(Function):

	def forward(logits, tau, k):
		sample = sampleOnehotCategorical(logits)
		save_for_backward(sample, logits, tau, k)
		return sample
	
	def backward(grad_output):
		sample, logits, tau, k = self.saved_objects
		logZ = logsumexp(logits)
		maxgumbel = getGumbel(logZ, k)
		tgumbels = getTruncatedGumbel(
			logits, k, sample, maxgumbel)
		gumbels = mergeGumbels(
			maxgumbel, tgumbels, sample)
		J = getSmaxJacobian(gumbels + logits).mean(0)
		return grad_output.matmul(J)
\end{lstlisting}

\section{Implementing Gradient Estimators with the Surrogate Loss Framework}
In this section, we consider an alternative framework for implementing  the gradient estimators presented in the main body. This framework is due to \citep{schulman2015gradient} and known as the surrogate loss framework. The key idea is that after the forward pass through a stochastic computation graph, all sampling decisions have been taken. Therefore, any gradient can be written as resulting from the differentiation of a surrogate objective in a deterministic computation graph. 

Our exposition in the main body only considered a simplified scenario with a single discrete random variable. Therefore, we present here two cases, involving a layer of multiple and a cascade of discrete random variables. These two cases are general, because any case can be reduced to either of these two or a combination of them. 

For ease of exposition, we again do not consider any direct dependence of $f$ on the parameters of interest $\theta$. The extension to this case is straight-forward and follows from basic calculus. 

We also introduce the following notation to denote the stop of gradient flow. For $X^* = \texttt{stop\_gradient}(X)$ indicates that the gradient flow is interrupted at $X$ and no gradient information is passed backward. 

\subsection{Parallel Case}
 Let $D^1, \ldots, D^m$ be a sequence of independent random variables. For $j \leq m$, let $D^j$ be a discrete random variable $D^j \in \{0, 1\}^n$ in a one-hot encoding, $\sum D_i^j = 1$, with distribution given by  $p_{\theta^j}(D^j) \propto \exp({D^j}^T \theta)$ where $\theta^j \in \R^n$. Further, let $S_{\tau}^j$ be defined analogously to equation \eqref{eq:coupling}. Given a continuously differentiable $f : \R^{mn} \to \R$, we wish to minimize
\begin{align}
    \label{eq:parallel} 
    \min_{\theta} 
    \mathbb{E}\left[ 
    f(D^1, \ldots D^m) 
    \right],
\end{align}
where the expectation is taken over all $m$ random variables. 
\newline
\newline
In this setting, $\REINFORCE$ can be computed by differentiating the following surrogate objective,
\begin{align}
f(D^{1*}, \ldots D^{m*}) \sum_{j=1}^m \log p_{\theta^j}(D^j)
\end{align}
\newline
\newline
In this setting, $\GS$ can be computed by differentiating the following surrogate objective,
\begin{align}
f(S_{\tau}^{1}, \ldots S_{\tau}^{m})  
\end{align}
\newline
\newline
In this setting, $\ST$ can be computed by differentiating the following surrogate objective,
\begin{align}
\sum_{j=1}^m 
\left(
\frac{
\partial
f(D^{1}, \ldots D^{m})
}
{
\partial 
D^{j}
}
\right)^*
\softmax_{\tau}(\theta^j)
\end{align}
\newline
\newline
In this setting, $\STGS$ can be computed by differentiating the following surrogate objective,
\begin{align}
\sum_{j=1}^m 
\left(
\frac{
\partial
f(D^{1}, \ldots D^{m})
}
{
\partial 
D^{j}
}
\right)^*
S_{\tau}^j
\end{align}
In this setting, $\GRMC{K}$ can be computed by differentiating the following surrogate objective,
\begin{align}
\sum_{j=1}^m 
\left(
\frac{
\partial
f(D^{1}, \ldots D^{m})
}
{
\partial 
D^{j}
}
\right)^*
\left[\frac{1}{K}
\sum_{k=1}^{K}
S_{\tau}^{jk}\right]
\end{align}

\subsection{Sequential Case}
Let $D^1, \ldots, D^m$ be a sequence of non-independent random variables. For $j \leq m$, let $D^j$ be a discrete random variable $D^j \in \{0, 1\}^n$ in a one-hot encoding, $\sum D_i^j = 1$, with distribution given by  $p_{\theta^j}(D^j) \propto \exp({D^j}^T \theta^j)$ where $\theta^j \in \R^n$. For $2 \leq j \leq m$, let $\theta^j = h(D^{j-1})$, where $h : \R^{n} \to \R^{n}$ is a continuously differentiable function. Given a continuously differentiable $f : \R^{mn}\to \R$, we wish to minimize
\begin{align}
\label{eq:sequential} 
\min_{\theta} 
\mathbb{E}\left[ 
f(D^1, \ldots, D^m)
\right]
\end{align}
In this setting, $\REINFORCE$ and $\GS$ can be computed by differentiating the surrogate objective given in the parallel case. \\
In this setting, $\ST$ can be computed by differentiating the following surrogate objective, 
\begin{align}
L_m &\coloneqq 
\left(\frac{\partial  f(D^1, \ldots, D^m)}{\partial D^m} \right)^*
\softmax_{\tau}(\theta^{m}) \\ 
L_j &\coloneqq 
\left(
\left(\frac{d L_{j+1}(D^1, \ldots, D^m)}{d D^{j}} \right)^*
+ 
\left(\frac{\partial  f(D^1, \ldots, D^m)}{\partial D^j} \right)^*
\right)
\softmax_{\tau}(\theta^{j})
\end{align}
In this setting, $\STGS$ can be computed by differentiating the following surrogate objective, 
\begin{align}
L_m &\coloneqq 
\left(\frac{\partial  f(D^1, \ldots, D^m)}{\partial D^m} \right)^*
\softmax_{\tau}(\theta^{m} + G^m) \\ 
L_j &\coloneqq 
\left(
\left(\frac{d L_{j+1}(D^1, \ldots, D^m)}{d D^{j}} \right)^*
+ 
\left(\frac{\partial  f(D^1, \ldots, D^m)}{\partial D^j} \right)^*
\right)
\softmax_{\tau}(\theta^{j}+ G^j)
\end{align}
In this setting, $\GRMC{K}$ can be computed by differentiating the following surrogate objective,
\begin{align}
L_m &\coloneqq 
\left(\frac{\partial  f(D^1, \ldots, D^m)}{\partial D^m} \right)^*
\left[\frac{1}{K}
\sum_{k=1}^{K}
\left(
\softmax_{\tau}(\theta^{m} + G^{mk})
\right)\right] \\ 
L_j &\coloneqq 
\left(
\left(\frac{d L_{j+1}(D^1, \ldots, D^m)}{d D^{j}} \right)^*
+ 
\left(\frac{\partial  f(D^1, \ldots, D^m)}{\partial D^j} \right)^*
\right)
\left[\frac{1}{K}
\sum_{k=1}^{K}
\left(
\softmax_{\tau}(\theta^{j}+ G^{jk})
\right)\right]
\end{align}

\section{Proofs for the Propositions}
\label{appendix:proofs}
In this section, we provide derivations for all the propositions given in the main body. 
\subsection{Proposition \ref{prop:gr_mse}}
\label{appendix:proofs_gr_mse}
The derivation is based on Jensen's inequality and the law of iterated expectations.
\begin{proof}
\begin{align}
\expect
\left[
\left\lVert \GR - \nabla_{\theta} \right\rVert^2
\right] 
&=
\expect
\left[
\left\lVert 
\expect\left[\STGS | D \right] - \nabla_{\theta} 
\right\rVert^2
\right] \\
&=
\expect
\left[
\left\lVert 
\expect\left[\STGS - \nabla_{\theta} | D \right] 
\right\rVert^2
\right] \\
&\leq
\label{eq:prop1_jensen}
\expect
\left[
\expect
\left[
\lVert 
\STGS - \nabla_{\theta} 
\rVert^2
| D \right] 
\right]  \\ 
&=
\label{eq:prop1_iterated}
\expect
\left[
\lVert 
\STGS - \nabla_{\theta} 
\rVert^2
\right]
\end{align}
\end{proof}
The inequality is strict whenever $\var\left[\STGS | D \right] > 0$, which is the case if $\tau < \infty$ and $\left|\theta_i\right| < \infty$ for all $i \leq n$. 
\subsection{Proposition \ref{prop:grmc_mse}}
\label{appendix:proofs_grmc_mse}
The derivation is based on Jensen's inequality and the linearity of expectations. For ease of exposition, denote by $\STGS \left(S^k| D\right)$ a particular realization of the ST-GS estimator for a given $D$. 
\begin{proof}
\begin{align}
\expect
\left[
\left\lVert \GRMC{K} - \nabla_{\theta} \right\rVert^2
\right] 
&=
\expect
\left[
\left\lVert 
\frac{1}{K}
\sum_{k=1}^K
\STGS \left(S^k| D\right) - \nabla_{\theta} 
\right\rVert^2
\right] \\
&\leq
\label{eq:prop2_jensen}
\expect
\left[
\frac{1}{K}
\sum_{k=1}^K
\lVert 
\STGS \left(S^k| D\right) - \nabla_{\theta} 
\rVert^2
\right]  \\ 
&=
\label{eq:prop2_linear}
\frac{1}{K}
\sum_{k=1}^K
\expect
\left[
\lVert 
\STGS \left(S^k| D\right) - \nabla_{\theta} 
\rVert^2
\right]  \\ 
&=
\expect
\left[
\lVert 
\STGS - \nabla_{\theta} 
\rVert^2
\right]	
\end{align}	
\end{proof}
The inequality is strict whenever $K > 1$ and $\var\left[\STGS | D \right] > 0$, which is the case if $\tau < \infty$ and $\left|\theta_i\right| < \infty$ for all $i \leq n$. 
\subsection{Proposition \ref{prop:grmc_var}}
\label{appendix:proofs_grmc_var}
The derivation is based on the law of total variance.
\begin{proof}
\begin{align}
\var\left[\GRMCMB{K}{B}\right]
 &= 
	\expect
\left[
\var
\left[
\GRMCMB{K}{B} | D 
\right]
\right]
+ 
\var
\left[
\expect
\left[
\GRMCMB{K}{B} | D 
\right]
\right] \\
 &= 
	\expect
\left[
\var
\left[
\frac{1}{B}\sum_{b=1}^B\GRMC{K}^b \middle| D 
\right]
\right]
+ 
\var
\left[
\expect
\left[
\frac{1}{B}\sum_{b=1}^B\GRMC{K}^b \middle| D 
\right]
\right] \\
 &= 
	\expect
\left[
\frac{1}{B}
\var
\left[
\GRMC{K} | D 
\right]
\right]
+ 
\var
\left[
\frac{1}{B}\sum_{b=1}^B
\expect
\left[
\GRMC{K} | D 
\right]
\right] \\
&= 
\frac{1}{B}
\expect
\left[
\frac{1}{K}
\var
\left[
\STGS | D 
\right]
\right]
+ 
\frac{1}{B}
\var
\left[
\expect
\left[
\GRMC{K} | D 
\right]
\right] \\
&= 
\frac{1}{BK}
\expect
\left[
\var
\left[
\STGS | D 
\right]
\right]
+ 
\frac{1}{B}
\var
\left[
\GR
\right]
\end{align}	
\end{proof}

\section{Experimental Details}
\label{appendix:exp_details}
\subsection{Unsupervised Parsing on ListOps}
\label{appendix:exp_nlp}
For our unsupervised parsing expeiment on ListOps, we use the basic version of the model described in \cite{choi2017unsupervised} with an embedding dimension and hidden dimension of $128$. We do not use the \emph{leaf-rnn}. We do not use the \emph{intra-attention module}. We do not use dropout, but set weight decay to be $1e-4$. Because our interest is in using this experiment primarily as a testbed to evaluate the effectiveness of different gradient estimators for this model at different temperatures and for trees of different depth, we use a very simple experimental set-up. We rely on stochastic gradient descent without momentum to train all models. We use grid search to determine an optimal learning rate from $\{0.1, 0.2, \ldots 1.0\}$ and set the temperature $\tau$ to be in $\{0.01, 0.1, 1.0\}$. We repeat five independent random runs at each setting and report the mean over the five runs. We train for ten epochs and set the batch size to be equal to the maximum sequence length $L$.

\subsection{Generative Modelling with Variational Auto-Encoders}
\label{appendix:exp_vae}
We trained variational auto-encoders with $n$-ary discrete random variables with values on the corners of the hypercube $\{-1, 1\}^{\log_2(n)}$. The model with arity $\{2, 4, 8, 16\}$ included $\{240, 120, 80, 60\}$ random variables respectively.

All models were optimized using stochastic gradient descent with momentum for 50000 steps on minibatches of size 20 and 200 respectively. Hyperparameters were randomly sampled and the best setting was selected from twenty independent runs. Learning rate and momentum were randomly sampled from $\{5, 6, \ldots 50\} \times 10^{-4}$ and $(0,1)$ respectively. We did not anneal the learning rate during training. For regularising the network, we used weight-decay, which was randomly sampled from $\{0, 10^{-1}, 10^{-2} \ldots, 10^{-6}\}$. The temperature was randomly sampled from $[0.1, 1.0]$ and not annealed throughout training. 

All models were evaluated on the validation and test set using the importance-weighted bound on the log-likelihood described in \cite{burda2015importance} with 5000 samples.

To estimate the variance of a gradient estimator in the VAE experiment we used 5000 randomly sampled mini-batches of size 20, for each of which we performed 100 independent forward passes and then computed the associated gradient for the parameters of the inference network. We then summed the variance to get a singe scalar measurement. 

To estimate the bias of a gradient estimator in the VAE experiment, we proceeded as above to approximate the expectation for a gradient estimator. We approximated the true gradient by following this procedure for the REINFORCE algorithm. 

To assess training speed, we measured the average number of iterations needed to achieve a prespecified loss threshold on the validation set. In particular, we ran multiple independent runs under the same experimental conditions for all gradient estimators. Among only runs that achieved the threshold within the total budget, we report the average number of iterations taken to cross the threshold.

\section{Additional Figures}
\label{appendix:add_fig}
\begin{figure*}[!h]
	\centering
	\begin{subfigure}[t]{0.24\textwidth}
		\centering
		\includegraphics[width=\textwidth]{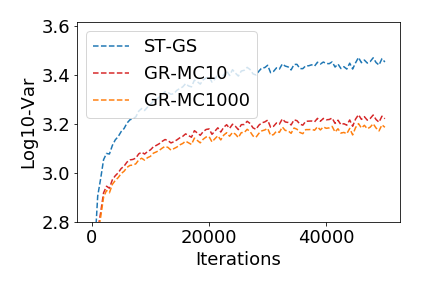}
	\end{subfigure}
	\hfill
	\begin{subfigure}[t]{0.24\textwidth}
		\centering
		\includegraphics[width=\textwidth]{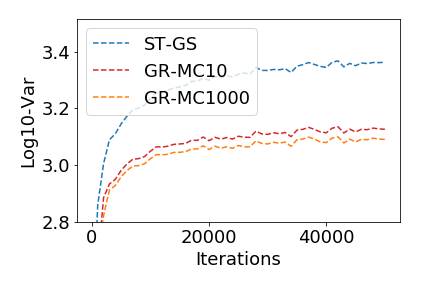}
	\end{subfigure}
	\hfill
	\begin{subfigure}[t]{0.24\textwidth}
		\centering
		\includegraphics[width=\textwidth]{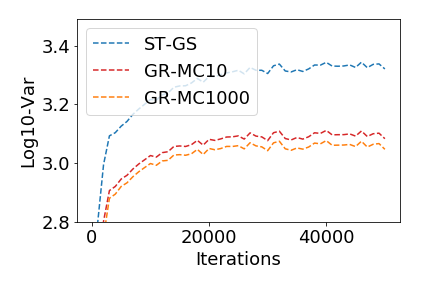}
	\end{subfigure}
	\hfill
	\begin{subfigure}[t]{0.24\textwidth}
		\centering
		\includegraphics[width=\textwidth]{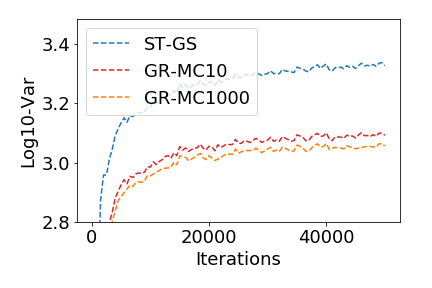}
	\end{subfigure}
	\\
	\begin{subfigure}[t]{0.24\textwidth}
		\centering
		\includegraphics[width=\textwidth]{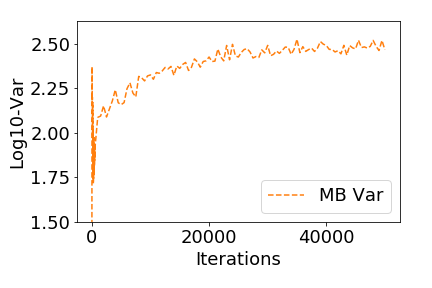}
		\caption{}		
	\end{subfigure}
	\hfill
	\begin{subfigure}[t]{0.24\textwidth}
		\centering
		\includegraphics[width=\textwidth]{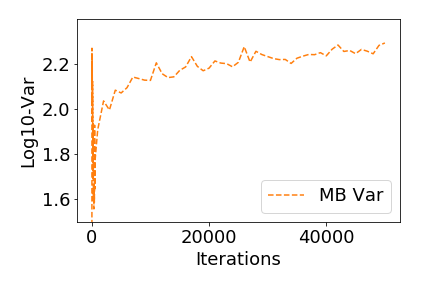}
		\caption{}		
	\end{subfigure}
	\hfill
	\begin{subfigure}[t]{0.24\textwidth}
		\centering
		\includegraphics[width=\textwidth]{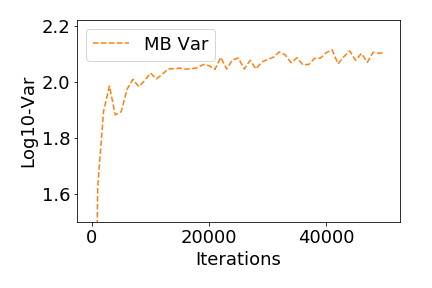}
		\caption{}
	\end{subfigure}
	\hfill	
	\begin{subfigure}[t]{0.24\textwidth}
		\centering
		\includegraphics[width=\textwidth]{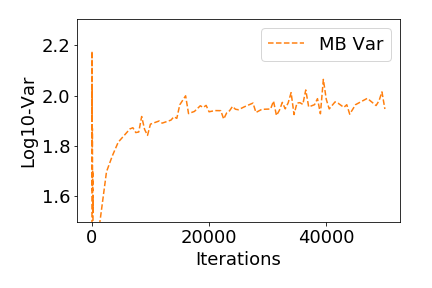}
		\caption{}
	\end{subfigure}
	\caption{Our estimator (GR-MC$K$) effectively reduces the variance over the entire training trajectory at all arities. The variance reduction compares favorable to the minibatch variance. Columns correspond to arities, i.e. (a) binary, (b) 4-ary, (c) 8-ary, (d) 16-ary. First row, log10-trace of MC covariance matrix for various gradient estimators over iterations. Second row, log10-trace of MB covariance matrix over iterations (same for all gradient estimators).}
	\label{fig:vae_additional}
\end{figure*}
\begin{figure*}[!h]
	\centering
	\begin{subfigure}[t]{0.24\textwidth}
		\centering
		\includegraphics[width=\textwidth]{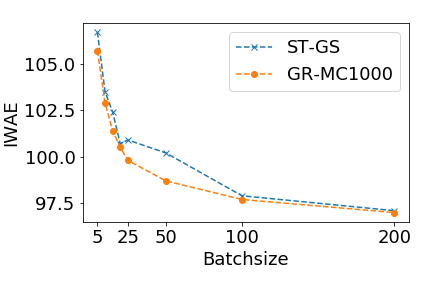}
	\end{subfigure}
	\hfill
	\begin{subfigure}[t]{0.24\textwidth}
		\centering
		\includegraphics[width=\textwidth]{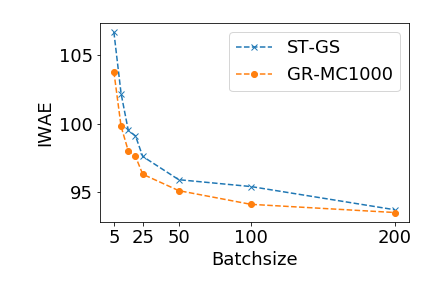}
	\end{subfigure}
	\hfill
	\begin{subfigure}[t]{0.24\textwidth}
		\centering
		\includegraphics[width=\textwidth]{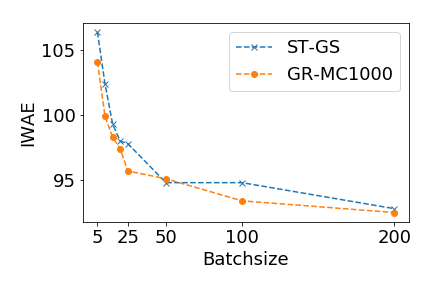}
	\end{subfigure}
	\hfill
	\begin{subfigure}[t]{0.24\textwidth}
		\centering
		\includegraphics[width=\textwidth]{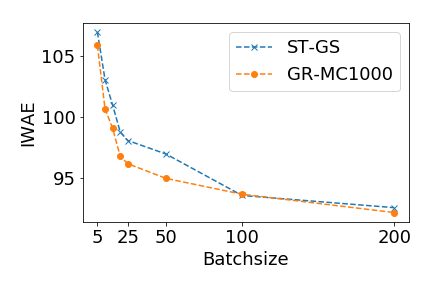}
	\end{subfigure}
	\\
	\begin{subfigure}[t]{0.24\textwidth}
		\centering
		\includegraphics[width=\textwidth]{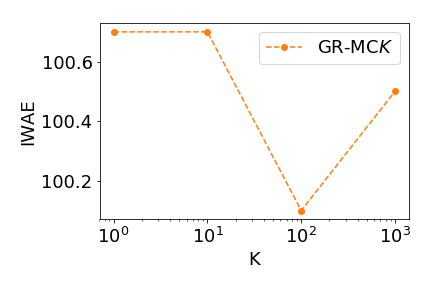}
		\caption{}		
	\end{subfigure}
	\hfill
	\begin{subfigure}[t]{0.24\textwidth}
		\centering
		\includegraphics[width=\textwidth]{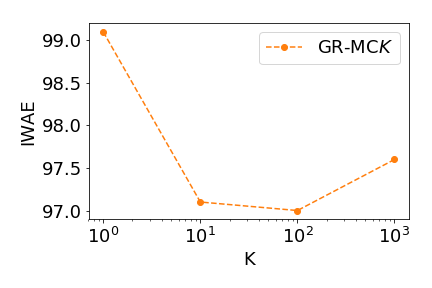}
		\caption{}		
	\end{subfigure}
	\hfill
	\begin{subfigure}[t]{0.24\textwidth}
		\centering
		\includegraphics[width=\textwidth]{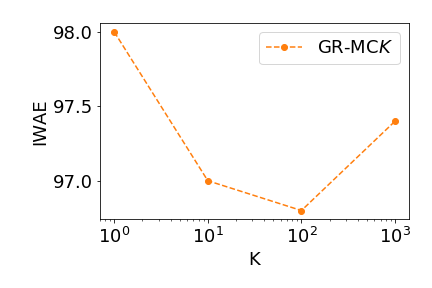}
		\caption{}		
	\end{subfigure}
	\hfill	
	\begin{subfigure}[t]{0.24\textwidth}
		\centering
		\includegraphics[width=\textwidth]{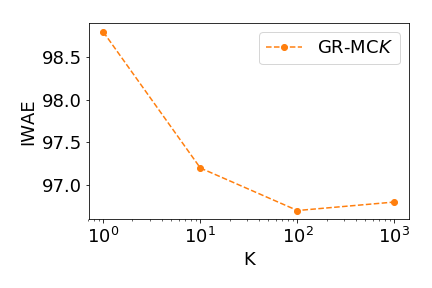}
		\caption{}		
	\end{subfigure}
	\caption{Increasing the number of Monte Carlo samples $K$ to reduce variance in gradient estimation tends to improve performance. The performance difference tends to be larger at smaller batch sizes. Columns correspond to arities, i.e. (a) binary, (b) 4-ary, (c) 8-ary, (d) 16-ary. First row, IWAE on test set for best validated model trained at various batch sizes. Second row, IWAE on test set for best validated model trained at various $K$ at batch size 20.}
\end{figure*}
\begin{table*}[t]
\ra{1.3}
\caption{Our estimator, GR-MC$K$, consistently achieves better performance across arities and batchsizes. The outperformance tends to be larger at smaller batchsizes. Best bound on the negative log-likelihood selected on the validation set from 20 independent runs at randomly searched hyperparameters.}
\label{table:vae:batchsizes}
\begin{center}
\begin{small}
\begin{sc}
\begin{tabular}{@{}llcrrrrrrrr@{}}
\toprule
& && \multicolumn{2}{c}{binary} &  \multicolumn{2}{c}{$4$-ary} &  \multicolumn{2}{c}{$8$-ary} &  \multicolumn{2}{c}{$16$-ary}\\
 \cmidrule(lr){4-5} \cmidrule(lr){6-7} \cmidrule(lr){8-9} \cmidrule(lr){10-11}
& Estimator && Valid. & Test &  Valid. & Test &  Valid. & Test &  Valid. & Test\\
\multirow{2}{*}{\shortstack[l]{batch-\\size 5}}& 
ST-GS && 107.7 & 106.7 & 107.8 & 106.7 & 107.5 & 106.4 & 108.1 & 107.0 \\
& GR-MC1000 && \textbf{106.7} & \textbf{105.7} & \textbf{104.7} & \textbf{103.8} & \textbf{105.1} & \textbf{104.1} & \textbf{107.0} & \textbf{105.9}\\
\midrule
\multirow{2}{*}{\shortstack[l]{batch-\\size 10}}& 
ST-GS && 104.4 & 103.5 & 103.2 & 102.2 & 103.5 & 102.4 & 104.1 & 103.1 \\
& GR-MC1000 && \textbf{103.7} & \textbf{102.9} & \textbf{100.8} & \textbf{99.8} & \textbf{100.9} & \textbf{99.9} & \textbf{101.8} & \textbf{100.7}\\
\midrule
\multirow{2}{*}{\shortstack[l]{batch-\\size 15}}& 
ST-GS && 103.4 & 102.4 & 100.4 & 99.5 & 100.3 & 99.3 & 101.9 & 101.0 \\
& GR-MC1000 && \textbf{102.3} & \textbf{101.4} & \textbf{99.0} & \textbf{98.0} & \textbf{99.2} & \textbf{98.3} & \textbf{100.2} & \textbf{99.1}\\
\midrule
\multirow{2}{*}{\shortstack[l]{batch-\\size 20}}& 
ST-GS && 101.5 & 100.7 & 100.0 & 99.1 & 99.0 & 98.0 & 99.8 & 98.8\\
& GR-MC1000 && \textbf{101.3} & \textbf{100.5} & \textbf{98.4} & \textbf{97.6} & \textbf{97.5} & \textbf{96.5} & \textbf{97.8} & \textbf{96.8}\\
\midrule
\multirow{2}{*}{\shortstack[l]{batch-\\size 25}}& 
ST-GS && 101.7 & 100.9 & 98.6 & 97.6 & 98.8 & 97.8 & 99.0 & 98.1 \\
& GR-MC1000 && \textbf{100.7} & \textbf{99.8} & \textbf{97.2} & \textbf{96.3} & \textbf{96.6} & \textbf{95.7} & \textbf{97.1} & \textbf{96.2}\\
\midrule
\multirow{2}{*}{\shortstack[l]{batch-\\size 50}}& 
ST-GS && 101.2 & 100.2 & 96.7 & 95.9 & \textbf{95.7} & \textbf{94.8} & 98.0 & 97.0 \\
& GR-MC1000 && \textbf{99.5} & \textbf{98.7} & \textbf{96.0} & \textbf{95.1} & 95.9 & 95.1 & \textbf{95.9} & \textbf{95.0}\\
\midrule
\multirow{2}{*}{\shortstack[l]{batch-\\size 100}}& 
ST-GS && 98.8 & 97.9 & 96.3 & 95.4 & 95.7 & 94.8 & \textbf{94.4} & \textbf{93.6} \\
& GR-MC1000 && \textbf{98.5} & \textbf{97.7} & \textbf{95.0} & \textbf{94.1} & \textbf{94.3} & \textbf{93.4} & 94.6 & 93.7 \\
\midrule
\multirow{2}{*}{\shortstack[l]{batch-\\size 200}}& 
ST-GS && 97.9 & 97.1 & 94.5 & 93.7 & 93.6 & 92.8 & 93.4 & 92.6\\
& GR-MC1000 && \textbf{97.8} & \textbf{97.0} & \textbf{94.3} & \textbf{93.5}  & \textbf{93.2} & \textbf{92.5} & \textbf{93.1} & \textbf{92.2}\\
\bottomrule
\end{tabular}
\end{sc}
\end{small}
\end{center}
\end{table*}

\end{document}